\documentclass[twoside,leqno,twocolumn]{article}

\usepackage[letterpaper]{geometry}

\usepackage{ltexpprt}
\usepackage{hyperref}

\usepackage{graphicx}
\usepackage{amsmath}
\usepackage{amssymb}

\usepackage{algorithm}
\usepackage{algorithmic}
\usepackage{multirow}
\usepackage{wrapfig}

\def \R {\mathbb{R}}

\def \x {\mathbf{x}}

\def \OO {\mathcal{O}}

\def \d {\mathbf{d}}

\def \T {\mathcal{T}}
\def \S {\mathcal{S}}

\newtheorem{thm}{Theorem}

\newtheorem{cor}{Corollary}

\begin{document}

\providecommand{\keywords}[1]
{
  \small	
  \textbf{\textit{Keywords---}} #1
}

\title{\Large Improved Knowledge Distillation via Full Kernel Matrix Transfer}
\author{Qi Qian\thanks{Alibaba Group, qi.qian@alibaba-inc.com} \and 
Hao Li\thanks{Alibaba Group, lihao.lh@alibaba-inc.com} \and 
Juhua Hu\thanks{University of Washington, juhuah@uw.edu}}

\date{}

\maketitle


\fancyfoot[R]{\scriptsize{Copyright \textcopyright\ 2022 by SIAM\\
Unauthorized reproduction of this article is prohibited}}





\begin{abstract} \small\baselineskip=9pt 
Knowledge distillation is an effective way for model compression in deep learning. Given a large model (i.e., teacher model), it aims to improve the performance of a compact model (i.e., student model) by transferring the information from the teacher. Various information for distillation has been studied. Recently, a number of works propose to transfer the pairwise similarity between examples to distill relative information. However, most of efforts are devoted to developing different similarity measurements, while only a small matrix consisting of examples within a mini-batch is transferred at each iteration that can be inefficient for optimizing the pairwise similarity over the whole data set. In this work, we aim to transfer the full similarity matrix effectively. The main challenge is from the size of the full matrix that is quadratic to the number of examples. To address the challenge, we decompose the original full matrix with Nystr{\"{o}}m method. By selecting appropriate landmark points, our theoretical analysis indicates that the loss for transfer can be further simplified. Concretely, we find that the difference between the original full kernel matrices between teacher and student can be well bounded by that of the corresponding partial matrices, which only consists of similarities between original examples and landmark points. Compared with the full matrix, the size of the partial matrix is linear in the number of examples, which improves the efficiency of optimization significantly. The empirical study on benchmark data sets demonstrates the effectiveness of the proposed algorithm. Code is available at \url{https://github.com/idstcv/KDA}.
\end{abstract}


\section{Introduction}
With the development of deep learning, neural networks make many computer vision tasks applicable for edge devices. Those devices often have limited computation and storage resources. Therefore, neural networks that contain a small number of FLOPS and parameters are preferred. Lots of efforts are devoted to improving the performance of neural networks with resource constraints~\cite{CourbariauxBD15,HintonVD15,SandlerHZZC18}. Among various developed strategies, knowledge distillation (KD) is a simple yet effective way to help train compact networks~\cite{HintonVD15}.

KD in deep learning aims to improve the performance of a small network (i.e., student) with the information from a large network (i.e., teacher). Given a teacher, various information can be transferred to regularize the training of the student network. \cite{HintonVD15} transfers the label information to smooth the label space of the student network. \cite{RomeroBKCGB14} and \cite{ZagoruykoK17} propose to transfer the information of intermediate layers to help training. \cite{YimJBK17} transfers the flow between layers as hints for student networks. \cite{ChenWZ18} improves the performance of metric learning with the rank information from teacher models. Recently, a number of methods are developed to transfer pairwise similarities between examples to student networks~\cite{LiuCLYHLD19, ParkKLC19, Peng_2019_ICCV, Tung_2019_ICCV} and different similarity functions have been investigated. In this work, we also focus on pairwise similarity matrix transfer but aim to improve the efficiency of transfer rather than similarity measurements.

Given two images $\x_i$ and $\x_j$, their similarity can be computed as $k(\x_i,\x_j) =\langle\phi(\x_i),\phi(\x_j)\rangle$
where $\phi(\cdot)$ is the corresponding kernel function that projects examples from the original space to a space that fits the problem better. In an appropriate space, a simple model (e.g., a linear model) can describe the data well~\cite{WilliamsS00}. Many efforts have been devoted to designing similarity functions but the strategy for effective transfer is less investigated. Note that the size of the full similarity matrix is of $n\times n$, where $n$ is the total number of examples. However, only a mini-batch of examples are often accessed at each iteration in a standard neural network training pipeline. Most of existing methods~\cite{LiuCLYHLD19, ParkKLC19, Peng_2019_ICCV, Tung_2019_ICCV} can transfer only the partial matrix (i.e., $r\times r$, where $r$ is the batch size) consisting of examples from a mini-batch, which can be inefficient for distilling the information of the full kernel matrix from teacher. Therefore, many KD methods that transfer pairwise similarity have to include an additional KD loss~\cite{HintonVD15} for the desired performance.

In this work, we propose to transfer the full kernel matrix between all examples from the teacher model to the student model efficiently. The main challenge comes from the size of the full matrix. With a large number of training examples, it becomes intractable to transfer the full matrix directly, especially for training neural networks with mini-batches. Therefore, we first apply the Nystr{\"{o}}m method~\cite{WilliamsS00} to obtain a low-rank approximation of the original matrix with landmark points. Then, we minimize the difference between the compact matrices that are calculated between original examples and landmark points, to transfer the information from the teacher effectively.

Compared with the full matrix of size $\OO(n^2)$, the size of compact one for transfer is only $\OO(n)$ in our method. Besides, the number of landmark points is significantly smaller than that of training examples and we can keep them in the network, which improves the efficiency of optimization. Considering that the selection of landmark points is important for approximating the original matrix, we propose to apply class centers as landmark points for better distillation. Our theoretical analysis shows that the difference between original matrices from teacher and student can be well bounded by that of the corresponding partial matrices. 
Interestingly, according to our analysis, the conventional KD method in \cite{HintonVD15} that transfers the label information can be considered as a special case of our proposal with sub-optimal landmark points. This observation provides a new perspective to understand the efficacy of KD. We conduct extensive experiments on benchmark data sets. Our method can achieve a satisfied performance without any additional KD loss, which confirms the effectiveness of transferring full kernel matrix. The main contributions of this work can be summarized as
\begin{itemize}
    \item Instead of transferring partial similarity matrices within only mini batches that can be inefficient for KD optimization, we propose to transfer the full kernel matrix from the teacher model to the student model.
    \item Considering that it is intractable to transfer the full matrix directly due to its large size and the mini-batch setting in training, we propose to obtain a low-rank approximation and minimize the difference between the compact matrices for transferring. 
    \item More importantly, we provide a theoretical guarantee that the difference between original full matrices from teacher and student can be well bounded by that of the corresponding compact matrices, which also provides a new perspective to understand the efficacy of conventional KD. 
    \item Our experiments including extensive ablation studies on benchmark data sets further demonstrate our theory and the effectiveness of our proposed method.
\end{itemize}

\section{Related Work}
\label{sec:related}

\subsection{Knowledge distillation} 

Knowledge distillation has a long history in ensemble learning and becomes popular for training small-sized neural networks~\cite{Ahn_2019_CVPR, ChenWZ18,HintonVD15,LiuCLYHLD19,mirzadeh2020improved,ParkKLC19, srinivas2018knowledge,  TianKI20,ZagoruykoK17}. Various algorithms have been developed  to transfer different information from the teacher model to the student model. For example, \cite{HintonVD15} considers the final output of the teacher model as the soft label and regularizes the similarity between the label distribution output from the student model and that of the soft label from the teacher model. \cite{ZagoruykoK17} transfers the attention maps from intermediate layers, which provides a way to explore more information from the teacher model. \cite{Ahn_2019_CVPR} proposes an information-theoretic framework for knowledge
transfer which formulates knowledge transfer as maximizing the mutual information between the teacher and the student networks.

Recently, a number of algorithms~\cite{LiuCLYHLD19, ParkKLC19, Peng_2019_ICCV, Tung_2019_ICCV} are proposed to transfer the pairwise similarity for knowledge distillation. However, they focus on developing similarity functions while only transferring the similarity information of pairs within a mini-batch, which can be inefficient for distilling. We, on the other hand, aim to transfer the full matrix directly to achieve a better performance. Furthermore, we provide a theoretical analysis to demonstrate the effectiveness of the proposed method. 

Besides classification, some methods are proposed for other tasks, e.g., detection~\cite{ChenCYHC17} and metric learning~\cite{ChenWZ18}. We focus on classification in this work while the proposed method can be easily extended to metric learning that aims to optimize the performance of the embedding layer.

\subsection{Nystr{\"{o}}m method} 

Nystr{\"{o}}m method is an effective algorithm to obtain a low-rank approximation for a Gram matrix~\cite{WilliamsS00}. Given a full Gram matrix, it tries to reconstruct the original one with the randomly sampled columns. The data points corresponding to the selected columns are denoted as landmark points. The approximation error can be bounded even with the randomly sampled landmark points. Later, researchers show that a delicate sampling strategy can further improve the performance~\cite{DrineasM05,KumarMT12,ZhangTK08}. \cite{DrineasM05} proposes to sample landmark points with a data-dependent probability distribution rather than the uniform distribution. \cite{KumarMT12} and \cite{ZhangTK08} demonstrate that using clustering centers as landmark points provides the best approximation among different strategies. 

Note that Nystr{\"{o}}m method is developed for unsupervised Gram matrix approximation while we can access the label information in knowledge distillation. In this work, we provide an analysis on the selection criterion of landmark points for Gram matrix transfer and develop a supervised strategy accordingly. Besides, Nystr{\"{o}}m method is useful only for a single Gram matrix approximation, while we aim to do transferring between two Gram matrices.

\section{Efficient Kernel Matrix Transfer}
\label{sec:method}
Given two images $\x_i$ and $\x_j$, the similarity between them can be measured with a kernel function as $k(\x_i, \x_j) = \langle \phi(\x_i),\phi(\x_j)\rangle$, where $\phi(\x_i)$ is a projection function that projects examples from the original space to a space better for the target task.

In this work, we consider a certain layer in a neural network as a projection function, where the output of that layer can be considered as $\phi(\x)$. We denote the student network as $\S$ and the teacher network as $\T$. The features output from a certain layer of $\S$ and $\T$ are referred as $f_\S(\x_i)=\x^i_\S$ and $f_\T(\x_i)=\x^i_\T$, respectively, and the index for the layer is omitted for brevity. Then, the similarity between two images $\x_i$ and $\x_j$ in the Gram matrix can be computed by
\begin{align*}
K_\S(\x_i,\x_j) = \langle f_\S(\x_i),f_\S(\x_j)\rangle=\x^{i\top}_\S\x^j_\S;\\
\quad K_\T(\x_i,\x_j) =  \langle f_\T(\x_i),f_\T(\x_j)\rangle=\x^{i\top}_\T\x^j_\T.
\end{align*}
Note that other complicated similarity functions can be applied between $f(\x_i)$ and $f(\x_j)$ as in existing methods~\cite{LiuCLYHLD19, ParkKLC19, Peng_2019_ICCV, Tung_2019_ICCV}.

Let $K_\S$ and $K_\T$ denote the $n\times n$ Gram matrices from the student and teacher networks, respectively, where $n$ is the total number of images. We aim to transfer the full Gram matrix from the teacher model to the student model. The corresponding loss for knowledge distillation with Gram matrix transfer can be written as
\begin{eqnarray}\label{eq:kda}
\ell_{\S,\T} = \|K_\S - K_\T\|_F = \|X_\S^\top X_\S - X_\T^\top X_\T\|_F
\end{eqnarray}
where $X_\S\in \R^{d_\S\times n}$ and $X_\T\in\R^{d_\T\times n}$ denote the representations of the entire data set output from the same certain layer of the student and teacher network.

Minimizing the loss directly is intractable due to the large size of the Gram matrix, especially for the conventional training pipeline in deep learning, where only a mini-batch of examples are accessible in each iteration. A straightforward way is to optimize only the random pairs in each mini-batch as in \cite{LiuCLYHLD19, ParkKLC19, Peng_2019_ICCV, Tung_2019_ICCV}. However, it can be sub-optimal and result in a slow optimization in transferring. Hence, we consider to decompose the Gram matrix and optimize the low-rank approximation in lieu of the full Gram matrix.

\subsection{Nystr{\"{o}}m Approximation}
Nystr{\"{o}}m method is prevalently applied to approximate the Gram matrix~\cite{DrineasM05,KumarMT12,WilliamsS00,ZhangTK08}. We briefly review it in this subsection.

Given a $n\times n$ Gram matrix $K$, we can first randomly shuffle columns and rewrite it as 
\[K = \left[\begin{array}{cc}W&K_{12}^\top\\K_{12}&K_{22}\end{array}\right]\]
where $W\in\R^{m\times m}$.
Then, a good approximation for $K$ can be obtained as $\widetilde{K} = CW^+C^\top$, where $C = \left[\begin{array}{c}W\\K_{12}\end{array}\right]\in\R^{n\times m}$ and $W^+$ denotes the pseudo inverse of $W$~\cite{WilliamsS00}.

Let $K_k$ denote the best top-$k$ approximation of Gram matrix $K$ and $k\leq m$. The rank-k approximation derived by the Nystr{\"{o}}m method can be computed as $\widetilde{K}_k = CW_k^+C^\top$,
where $W_k$ denotes the best top-$k$ approximation of $W$ and $W_k^+$ is the corresponding pseudo inverse. The performance of the approximation can be demonstrated by the following theorem.
\begin{thm}\cite{KumarMT12}\label{thm:1}
Let $\widetilde{K}_k$ denote the rank-$k$ Nystr{\"{o}}m approximation with $m$ columns that are sampled uniformly at random without replacement from $K$. We have
$\|K - \widetilde{K}_k\|_F\leq \|K-K_k\|_F+ \varepsilon;\quad \varepsilon=\OO(1/m^{1/4})\|K\|_F$.
\end{thm}

The examples corresponding to the selected columns in $C$ are referred as landmark points. Theorem~\ref{thm:1} shows that the approximation is applicable even with randomly sampled landmark points.

\subsection{Gram Matrix Transfer}

With the low-rank approximation, the Gram matrix from a certain layer of the student and teacher network can be computed as
\[\widetilde{K}_\S^k = C_\S W_{k\S}^+C_\S^\top; \quad \widetilde{K}_\T^k = C_\T W_{k\T}^+C_\T^\top\]
Compared with the original Gram matrix, the partial matrix $C$ has significantly less number of terms. Let $D_\S\in\R^{d_\S\times m}$ and $D_\T\in\R^{d_\T\times m}$ denote the landmark points for the student and teacher Gram matrices, then we have 
\[C_\S = X_\S^\top D_\S,\ C_\T = X_\T^\top D_\T;\]
\[W_\S = D_\S^\top D_\S,\ W_\T = D_\T^\top D_\T.\]
With Theorem~\ref{thm:1}, it is easy to show that transferring the compact matrix is up-bounding the distance between original Gram matrices from the teacher and the student.
\begin{cor}\label{cor:1}
With the Nystr{\"{o}}m approximation, we can bound the loss in Eqn.~\ref{eq:kda} by
\[\|K_\S - K_\T\|_F\leq 2\varepsilon+\|C_\S W_{k\S}^+C_\S^\top - C_\T W_{k\T}^+C_\T^\top\|_F\]
\end{cor}
In Corollary~\ref{cor:1}, the partial Gram matrix will be regularized with the pseudo inverse of $W$. The computational cost of obtaining pseudo inverse is expensive and it can introduce additional noise when the feature space of the student is unstable in the early stage of training. Besides, it still optimizes the difference between two $n\times n$ matrices.

By selecting ingenious landmark points, we can bound the original loss in Eqn.~\ref{eq:kda} solely with $C_\S$ and $C_\T$ as in the following theorem. 

\begin{thm}\label{thm:3}
Assuming that $C_\S$ and $C_\T$ are bounded by a constant $c$ as $\|C_\S\|_F,\|C_\T\|_F\leq c$ and the smallest eigenvalues in $W_{\S}$ and $W_{\T}$ are larger than $1$, we can bound the loss in Eqn.~\ref{eq:kda} by
\[\|K_\S - K_\T\|_F\leq 2\varepsilon+\OO(\|C_\S  - C_\T \|_F)\]
\end{thm}
\begin{proof}
\begin{align*}
&\|K_\S - K_\T\|_F = \|K_\S - \tilde{K}_\S^k - (K_\T - \tilde{K}_\T^k)+\tilde{K}_\S^k - \tilde{K}_\T^k\|_F\\
&\leq  \|K_\S - \tilde{K}_\S^k\|_F+\|(K_\T - \tilde{K}_\T^k)\|_F+\|\tilde{K}_\S^k - \tilde{K}_\T^k\|_F\\
&\leq 2\varepsilon+\|C_\S W_{k\S}^+C_\S^\top - C_\T W_{k\T}^+C_\T^\top\|_F\\
&\leq 2\varepsilon+c\|W_{k\S}^+C_\S^\top - W_{k\T}^+C_\T^\top\|_F+c(\|W_{k\T}^+\|_F\|C_\S - C_\T\|_F)\\
&\leq 2\varepsilon+c^2\|W_{k\S}^+ - W_{k\T}^+\|_F+2c^2\|C_\S - C_\T\|_F
\end{align*}

Then, we want to show that $\|W_{k\S}^+ - W_{k\T}^+\|_F\leq \|C_\S - C_\T\|_F$. Note that $\|W_{k\S} - W_{k\T}\|_F\leq \|C_\S - C_\T\|_F$ due to the fact that $W_{k}$ is a partial matrix from $C$. We can prove $\|W_{k\S}^+ - W_{k\T}^+\|_F\leq \|W_{k\S} - W_{k\T}\|_F$ instead. Let
\[W_{k\S} = \sum_{j=1}^k \alpha_{j}\mu_{j}\mu_{j}^\top;\quad W_{k\T} = \sum_{j=1}^k \beta_{j}\nu_{j}\nu_{j}^\top\]
where $\alpha_1\geq \cdots\geq \alpha_k>1$ and $\beta_1\geq \cdots\geq \beta_k>1$, and $M\in\R^{k\times k}$, where $M_{i,j} = \langle \mu_i\mu_{i}^\top, \nu_j\nu_j^\top\rangle$.
We have
\[\|W_{k\S} - W_{k\T}\|_F^2 = \sum_j \alpha_j^2+\sum_j\beta_j^2-2\sum_{s,t}\alpha_s\beta_tM_{s,t}\]
\[\|W_{k\S}^+ - W_{k\T}^+\|_F^2 = \sum_j \frac{1}{\alpha_j^{2}}+\sum_j\frac{1}{\beta_j^{2}}-2\sum_{s,t}\frac{1}{\alpha_s\beta_t}M_{s,t}\]
So
\begin{align}
&\|W_{k\S}^+ - W_{k\T}^+\|_F^2 - \|W_{k\S} - W_{k\T}\|_F^2 \nonumber\\
&= \sum_j \frac{1}{\alpha_j^{2}}+\sum_j\frac{1}{\beta_j^{2}} - \sum_j \alpha_j^2+\sum_j\beta_j^2\nonumber\\
&+2\sum_{s,t}(\alpha_s\beta_t - \frac{1}{\alpha_s\beta_t})M_{s,t} \nonumber
\end{align}
According to the definition of $M$, we have
\[\forall i,\quad \sum_j^kM_{i,j} = \mu_i(\sum_j \nu_j\nu_j^\top )\mu_i^\top\leq 1;\quad  \forall j, \quad \sum_i^k M_{i,j}\leq 1\]
Since $M$ is a doubly stochastic matrix, we can show that
\[\sum_{s,t}(\alpha_s\beta_t - \frac{1}{\alpha_s\beta_t})M_{s,t}\leq \sum_{s} (\alpha_s\beta_s - \frac{1}{\alpha_s\beta_s})\]
It can be proved by contradiction. If the optimal solution $M^*$ has a larger result than the R.H.S., we can denote the first column index of the non-zero off-diagonal element as $j>i$ (i.e., $M_{i,j}$), and the corresponding row index as $k>i$ (i.e., $M_{k,i}$). Let $A_{i,j} = \alpha_i\beta_j - \frac{1}{\alpha_i\beta_j}$ and we have
\begin{align*}
&A_{i,j} +  A_{k,i} - A_{i,i} - A_{k,j} \\
&= \alpha_i\beta_j - \frac{1}{\alpha_i\beta_j}+\alpha_k\beta_i - \frac{1}{\alpha_k\beta_i} \\
&- (\alpha_i\beta_i - \frac{1}{\alpha_i\beta_i})-(\alpha_k\beta_j - \frac{1}{\alpha_k\beta_j})\\
&=-(\alpha_i - \alpha_k)(\beta_i-\beta_j) +(\frac{1}{\alpha_i} - \frac{1}{\alpha_k})(\frac{1}{\beta_i}-\frac{1}{\beta_j})\\
&=-(\alpha_i - \alpha_k)(\beta_i-\beta_j)+\frac{(\alpha_i - \alpha_k)(\beta_i-\beta_j)}{\alpha_i\alpha_k\beta_i\beta_j}< 0
\end{align*} 
It shows that the assignment with the diagonal element can achieve a larger result than the optimal assignment, which contradicts the assumption.

With the optimal results from the assignment of $M$, we obtain that
\[\|W_{k\S}^+ - W_{k\T}^+\|_F^2 - \|W_{k\S} - W_{k\T}\|_F^2\leq \sum_j^k (\frac{1}{\alpha_j} - \frac{1}{\beta_j})^2 - (\alpha_j - \beta_j)^2\]
For each term, it is easy to show that
\[\forall j,\quad (\frac{1}{\alpha_j} - \frac{1}{\beta_j})^2 - (\alpha_j - \beta_j)^2\leq 0\]
Therefore,
\[\|W_{k\S}^+ - W_{k\T}^+\|_F\leq \|W_{k\S} - W_{k\T}\|_F\leq \|C_\S - C_\T\|_F\]\end{proof}

Theorem~\ref{thm:3} illustrates that minimizing the difference between the partial Gram matrices from the teacher and the student using landmark points can transfer the original Gram matrix from the teacher model effectively. Note that the partial matrices have the size of $n\times m$, where $m$ is the number of landmark points. When $m\ll n$, landmark points $D_\S$ and $D_\T$ can be kept in GPU memory as parameters of the loss function for optimization, which is much more efficient than transferring the original full Gram matrix. In addition, the simplified formulation depends on the property of eigenvalues from landmark points. Therefore, an appropriate selection is crucial for the efficient transfer. We elaborate our strategy in the next subsection.

\subsection{Landmark Selection}
\label{sec:land}
We consider a strategy that obtains a single landmark point for each class, which means $m=L$ for a $L$-class classification problem. The desired landmark points should capture the distribution of examples well while keeping a sufficient diversity between each other for large eigenvalues. With the setting that each class has one landmark point, we can demonstrate the selection criterion as follows.

\begin{thm}\label{thm:5}
Given an arbitrary pair $(i,j)$, let $\d_\S^i$ and $\d_\T^i$ denote the corresponding landmark points for the $i$-th example in the space of student and teacher network, respectively. Assuming the norm of $\x^i_\S,\x^i_\T,\x^j_\S, \x^j_\T$ are bounded by $e$, we have
\begin{align}\label{eq:k}
\|K_\S  - K_\T\|_F\leq
\underbrace{ne\sum_i \|\x^{i\top}_\S-\d^{i\top}_\S\|+ne\sum_i\|\x^{i\top}_\T-\d^{i\top}_\T\|}_{A}\\
+\underbrace{\sum_{i,j} \|\d^{i\top}_\S\x^j_\S - \d^{i\top}_\T\x^j_\T\|}_{B} \nonumber
\end{align}
\end{thm}
\begin{proof}
For an arbitrary pair, we have
\begin{align*}
&\|K_\S(\x_i,\x_j)  - K_\T(\x_i,\x_j)\| = \|\x^{i\top}_\S\x^j_\S- \x^{i\top}_\T\x^j_\T\|\\
& =  \|(\x^{i\top}_\S-\d^{i\top}_\S+\d^{i\top}_\S)\x^j_\S- (\x^{i\top}_\T-\d^{i\top}_\T+\d^{i\top}_\T)\x^j_\T\|\\
& \leq \|(\x^{i\top}_\S-\d^{i\top}_\S)\x^j_\S\|+ \|(\x^{i\top}_\T-\d^{i\top}_\T)x^j_\T\|\\
&+\|\d^{i\top}_\S\x^j_\S - \d^{i\top}_\T\x^j_\T\|\\
&\leq e\|\x^{i\top}_\S-\d^{i\top}_\S\|+e\|\x^{i\top}_\T-\d^{i\top}_\T\|+\|\d^{i\top}_\S\x^j_\S - \d^{i\top}_\T\x^j_\T\|
\end{align*}
We obtain the result by accumulating over all pairs.
\end{proof}

According to Theorem~\ref{thm:5}, we can find that the transfer loss comes from two aspects. Term $A$ in Eqn.~\ref{eq:k} contains the distance from each example to its corresponding landmark point. Since the corresponding landmark point for $\x_\S^i$ can be obtained by $\arg\min_{1\le l\le L}\{\|\x^{i\top}_\S-\d^{l\top}_\S\|\}$, we can rewrite the problem of minimizing the original term as
\[\min\sum_i\min_l\{\|\x^{i\top}_\S-\d^{l\top}_\S\|\}\]
Apparently, this objective is a standard clustering problem. Unlike the conventional Nystr{\"{o}}m method, which is often in an unsupervised learning setting, we can access the label information in knowledge distillation. When we set the number of clusters to be the number of classes, the landmark point becomes the center in each class and can be computed by averaging examples within the class as
\begin{eqnarray}\label{eq:updatec}
\d_\S^l = \frac{1}{n_l}\sum_{i:y_i=l}\x_\S^i;\quad \d_\T^l = \frac{1}{n_l}\sum_{i:y_i=l}\x_\T^i
\end{eqnarray}
where $y_i$ is the class label of the $i$-th example and $n_l$ denotes the number of examples in the $l$-th class. In addition, class centers keep the diversity from original examples which can have the large eigenvalues as required by Theorem~\ref{thm:3}. We also verify the assumption in the experiments.

Term $B$ from Eqn.~\ref{eq:k}, in fact, indicates the difference between the student and teacher Gram matrices defined by the landmark points that is consistent as in Theorem~\ref{thm:3}. 

With landmark points $D_\S$ and $D_\T$ obtained from optimizing the term $A$, we can formulate the Knowledge Distillation problem by transferring Approximated Gram matrix (KDA) as
\begin{align}\label{eq:kloss}
\min_{X_\S} \ell_{\mathrm{KDA}}(X_\S^\top D_\S- X_\T^\top D_\T)
\end{align}
where $\ell_{\mathrm{KDA}}(X_\S^\top D_\S- X_\T^\top D_\T)= \sum_{i=1,l=1}^{i=n, l=L}\ell(\d^{l\top}_\S\x^i_\S - \d^{l\top}_\T\x^i_\T)$ and we adopt $\ell(\cdot)$ as the smoothed $L_1$ loss for the stable optimization as \[\ell(z) = \left\{\begin{array}{cc}|z|-0.5&z>1\\0.5z^2&o.w.\end{array}\right.\]

With the KDA loss, we propose a novel knowledge distillation algorithm that works in an alternative manner. In each iteration, we first compute the landmark points with features of examples accumulated from the last epoch by Eqn.~\ref{eq:updatec}. Then, the KDA loss defined by the fixed landmark points in Eqn.~\ref{eq:kloss} will be optimized along with a standard cross-entropy loss for the student. The proposed algorithm is summarized in Alg.~\ref{alg:kda}. Since at least one epoch will be spent on collecting features for computing landmark points, we will minimize the KDA loss after $H$ epochs of training, where $H\geq 1$.

\begin{algorithm}[t]
   \caption{\textbf{K}nowledge \textbf{D}istillation by \textbf{A}pproximated Kernel Matrix Transfer (KDA)}
   \label{alg:kda}
\begin{algorithmic}
   \STATE {\bfseries Input:} Data set $\{\x_i\}$, a student model $\S$, a teacher model $\T$, total epochs $T$, warm-up epochs $H$
   \STATE Initialize $\{\d_\S^0\}=\emptyset$, $\{\d_\T^0\} = \emptyset$
   \FOR{$t=1$ {\bfseries to} $H$}
   \STATE Optimize $\S$ without KDA loss
   \STATE Compute $\{\d_\S^t\}$, $\{\d_\T^t\}$ as in Eqn.~\ref{eq:updatec}
   \ENDFOR
   \FOR{$t=H+1$ {\bfseries to} $T$}
   \STATE Optimize $\S$ with KDA loss defined on $\{\d_\S^{t-1}\}$, $\{\d_\T^{t-1}\}$
   \STATE Compute $\{\d_\S^t\}$, $\{\d_\T^t\}$ as in Eqn.~\ref{eq:updatec}
   \ENDFOR
   \STATE {\bfseries Output:} $\S$
\end{algorithmic}
\end{algorithm}

\subsection{Connection to Conventional KD}\label{sec:connection}
In the conventional KD method~\cite{HintonVD15}, only the outputs from the last layer in the teacher model are adopted for the student. By setting an appropriate parameter, \cite{HintonVD15} illustrates that the loss function for KD can be written as
\[\ell(X_\S) = \sum_i\|\x_\S^i - \x_\T^i\|^2\]
where $\x_\S^i, \x_\T^i\in\R^{C}$ denote the logits before the SoftMax operator. With the identity matrix $I$, the equivalent formulation is
\[\ell(X_\S) = \|X_\S^\top I - X_\T^\top I\|_F^2\]
Compared to the KDA loss in Eqn.~\ref{eq:kloss}, the conventional KD can be considered as applying one-hot label vectors as landmark points to transfer the Gram matrix of the teacher network. However, it lacks the constraints on the similarity between each example and its corresponding landmark point as illustrated in Theorem~\ref{thm:5}, which may degenerate the performance of knowledge distillation.

\section{Experiments}
\label{sec:exp}
We conduct experiments on two benchmark data sets to illustrate the effectiveness of the proposed KDA algorithm. We employ ResNet-34 (denoted as R34)~\cite{HeZRS16}  as the teacher network. ResNet-18 (denoted as R18), ResNet-18-0.5 (denoted as R18-0.5) and ShuffleNetV2 (denoted as SN)~\cite{MaZZS18} are adopted as student networks, where ResNet-18-0.5 denotes ResNet-18 with a half number of channels. We apply the standard stochastic gradient descent (SGD) with momentum to train the networks. Specifically, we set the size of mini-batch to $256$, momentum to $0.9$ and weight decay to $5e$-$4$ in all experiments. The student models are trained with $120$ epochs. The initial learning rate is $0.1$ and cosine decay is adopted with $H=5$ epochs for warm-up. All experiments are implemented on two V100 GPUs.

Three baseline knowledge distillation methods are included in the comparison
\begin{itemize}
\item KD~\cite{HintonVD15}: a conventional knowledge distillation method that constrains the KL-divergence between the output label distributions of the student and teacher networks.
\item AT~\cite{ZagoruykoK17}: a method that transfers the information from intermediate layers to accelerate the training of the student network.
\item RKD~\cite{ParkKLC19}: a recent representative work that regularizes the similarity matrices within a mini-batch between the student and teacher networks. It adopts the same similarity function as in our method. Although different ways were used to calculate the similarity (e.g., L2 in \cite{LiuCLYHLD19,ParkKLC19}, normalized inner product in \cite{Tung_2019_ICCV} and RBF kernel in \cite{Peng_2019_ICCV}), they provide similar performance.
\end{itemize}
Every algorithm will minimize the combined loss from both the distillation and the standard cross entropy loss for classification. For RKD, we transfer the features before the last fully-connected (FC) layer for comparison. Note that AT transfers the attention map of the teacher, so we adopt the feature before the last pooling layer for distillation. Besides, we let ``Baseline'' denote the method that trains the student without information from the teacher. Our method is referred as ``KDA''. We search the best parameters for all methods in the comparison and keep the same parameters for different experiments.

\subsection{Ablation Study}

We perform the ablation study on CIFAR-100~\cite{Krizhevsky2009Learning}. This data set contains $100$ classes, where each class has $500$ images for training and $100$ for test. Each image is a color image with size of $32\times 32$. 

In this subsection, we set ResNet-34 as the teacher and ResNet-18 as the student. During the training, each image is first padded to be $40\times 40$, and then we randomly crop a $32\times 32$ image from it. Besides, random horizontal flipping is also adopted for data augmentation.

\subsubsection{Effect of Landmark Points}
First, we evaluate the strategy for generating landmark points. As illustrated in Corollary~\ref{cor:1}, the randomly selected landmark points can achieve a good performance. So we compare the KDA with class centers to that with random landmark points in Fig.~\ref{fig:lp}. In this experiment, we adopt the features before the last FC layer for transfer.

\begin{figure}[h]
\centering
\begin{minipage}{1.6in}
\centering
\includegraphics[width=1.6in]{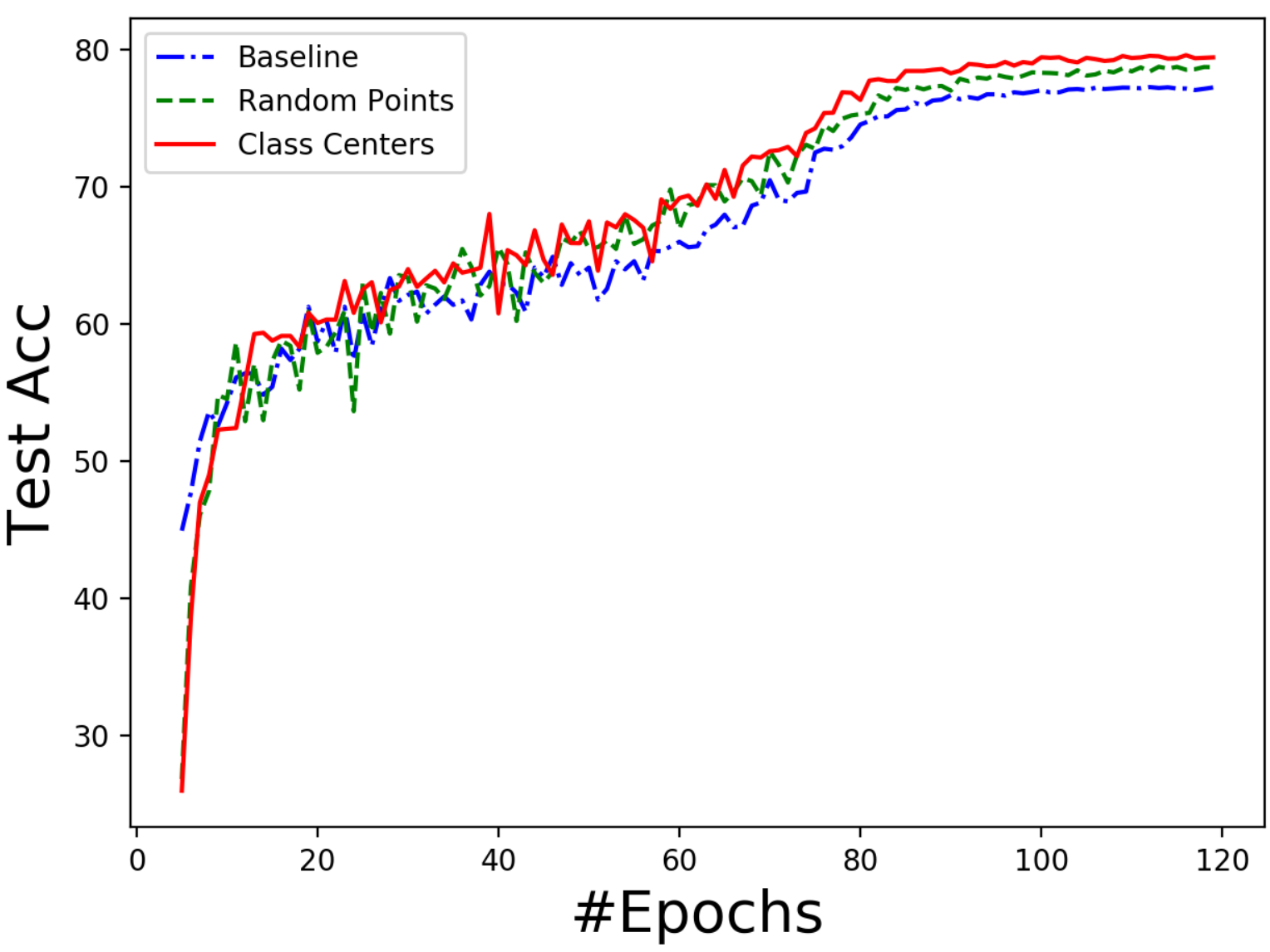}
\mbox{\footnotesize (a) Overall Comparison}
\end{minipage}
\begin{minipage}{1.6in}
\centering
\includegraphics[width=1.6in]{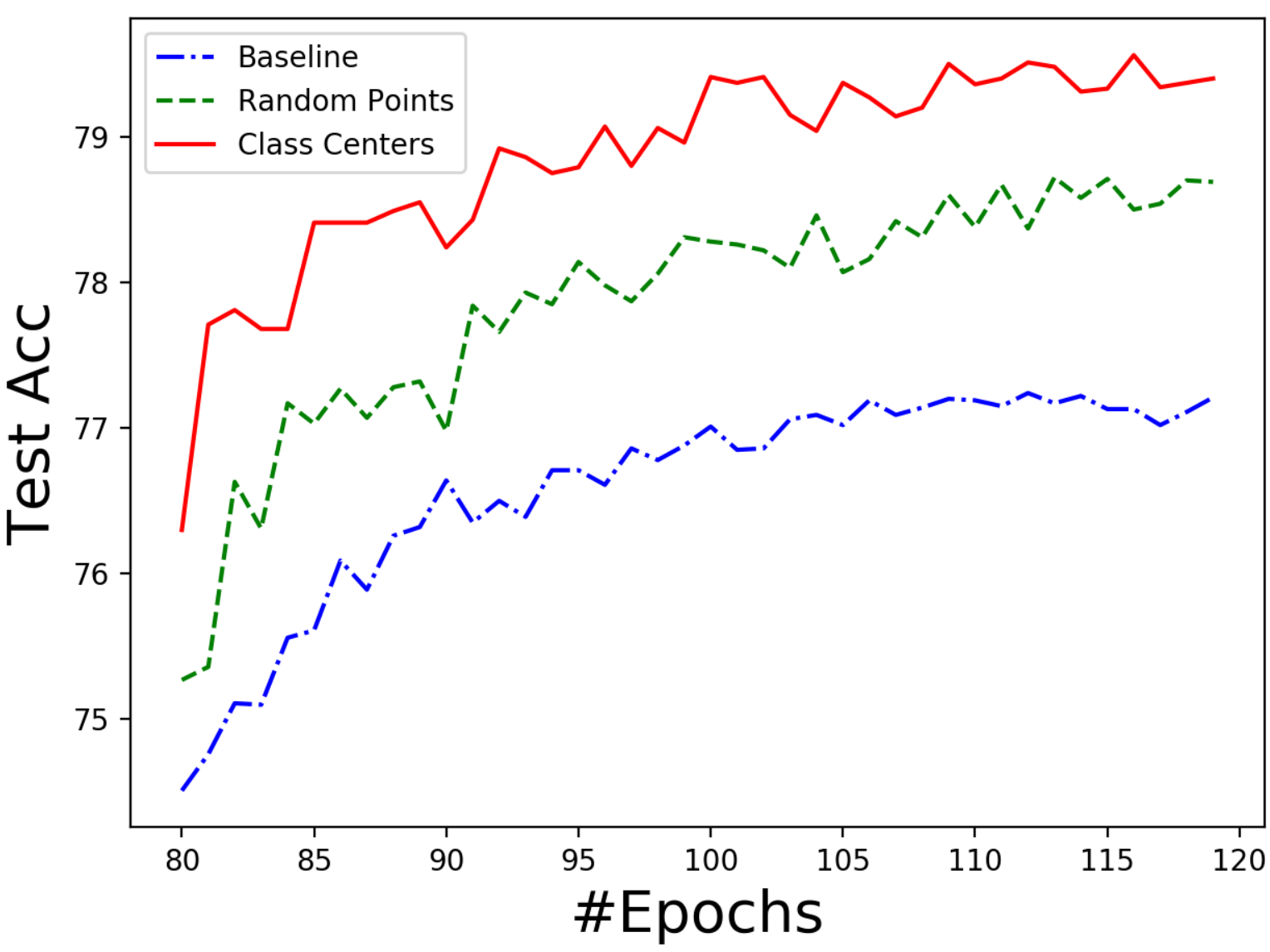}
\mbox{\footnotesize (b) Zoom In}
\end{minipage}
\caption{Comparison of landmark points selection.\label{fig:lp}}
\end{figure}

From Fig.~\ref{fig:lp}, we can observe that with landmark points, two variants of KDA perform significantly better than the baseline. It demonstrates that Gram matrix is informative for training student models and transferring full matrix from teacher can help improve the performance of student. In addition, KDA with class centers as the landmark points shows the best performance among different methods, which confirms the criterion suggested in Theorem~\ref{thm:5}. We will use class centers as landmark points in the remaining experiments. 

\subsubsection{Effect of Full Matrix Transfer}

Then, we compare the difference between Gram matrices from a teacher and its student models. The performance of transferring a Gram matrix is measured by $\|K_\S-K_\T\|_F/\|K_\T\|_F$, which calculates the faction of information that has not been transferred. We investigate features from two layers in the comparison: the one before the last FC layer and that after the FC layer. The transfer performance of different layers are illustrated in Fig.~\ref{fig:kerr} (a) and (b), respectively.

\begin{figure}[ht]
\centering
\begin{minipage}{1.6in}
\centering
\includegraphics[width=1.6in]{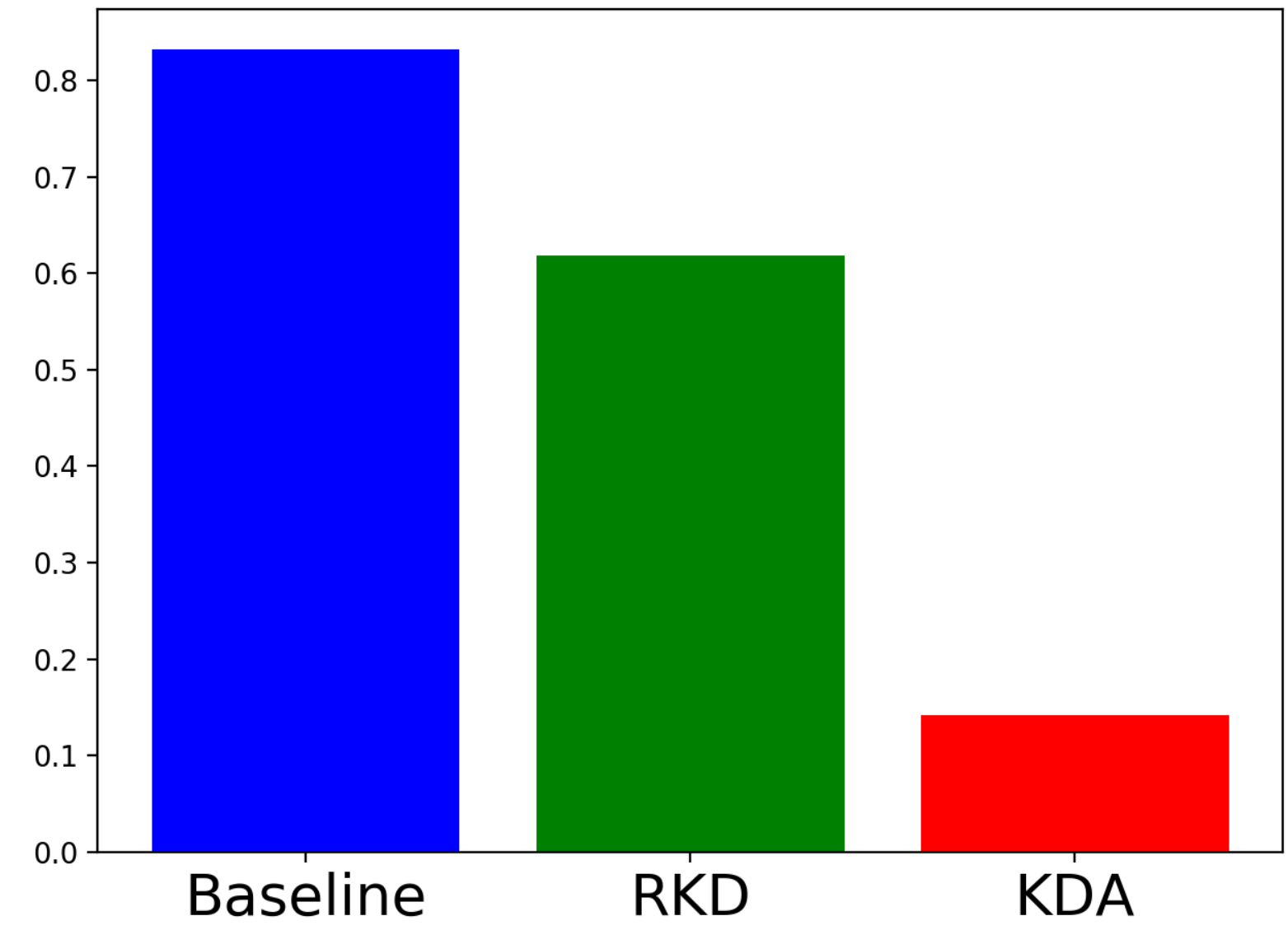}
\mbox{\footnotesize (a) Layer before FC}
\end{minipage}
\begin{minipage}{1.6in}
\centering
\includegraphics[width=1.6in]{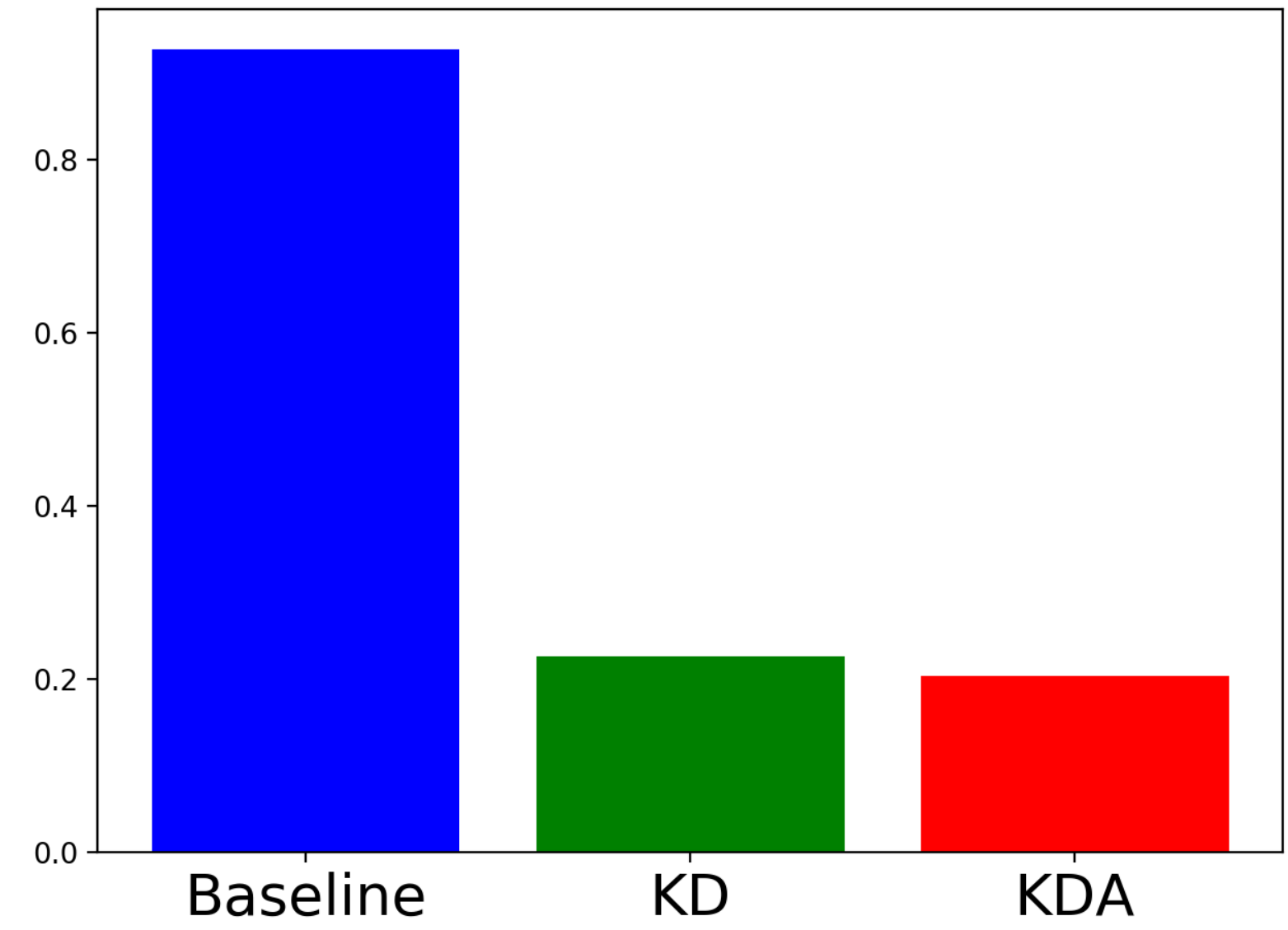}
\mbox{\footnotesize (b) Layer after FC}
\end{minipage}
\caption{Comparison of kernel transfer performance measured by $\|K_\S-K_\T\|_F/\|K_\T\|_F$ (lower is better).\label{fig:kerr}}
\end{figure}

Fig.~\ref{fig:kerr} (a) compares the performance of kernel transfer. First, it is obvious that both RKD and KDA are better than the baseline. It indicates that minimizing the difference between Gram matrices can effectively transfer appropriate information from the teacher. Second, RKD transfers the similarity matrix defined by examples in a mini-batch only and shows a larger transfer loss than KDA. Considering the massive number of pairs, optimizing with all of these pairs in RKD is intractable. Note that the number of pairs can be up to $\OO(n^2)$ while the number of pairs in a mini-batch is only $\OO(r^2)$, where $r$ is the size of a mini-batch. To visit all pairs only once, it requires at least $\OO(n/r)$ epochs.

On the contrary, the loss from KDA is only about $23\%$ of that from RKD. KDA optimizes the partial Gram matrix with landmark points and the total number of pairs is linear in that of original examples. Due to a small number of landmark points, the partial matrix $C$ is much more compact than the original one. For example, there are $50,000$ examples in CIFAR-100. When applying $100$ landmark points for distillation, the partial matrix contains $0.2\%$ terms of the original one. Besides, since we keep class centers in the memory as the parameters of the loss function, the full Gram matrix can be approximated in a single epoch. Therefore, SGD can optimize the KDA loss sufficiently.

Then, we compare the performance of transfer after the last FC layer as shown in Fig.~\ref{fig:kerr} (b). For KDA, we compute the Gram matrix with features before the SoftMax operator. From the comparison, we can observe that both of KD and KDA have much less transfer loss than the baseline. As illustrated in the discussion of ``Connection to Conventional KD'', conventional KD is equivalent to transferring the partial Gram matrix with one-hot landmark points. Therefore, it can reduce the difference between teacher and student effectively. However, the landmark points adopted by KD fail to satisfy the property illustrated in Theorem~\ref{thm:5}. By equipping class centers as landmark points, KDA can further reduce the transfer loss from $0.23$ in KD to $0.2$, which confirms the effectiveness of transferring full Gram matrix with appropriate landmark points.

\begin{figure}[!ht]
\centering
\includegraphics[width=2.4in]{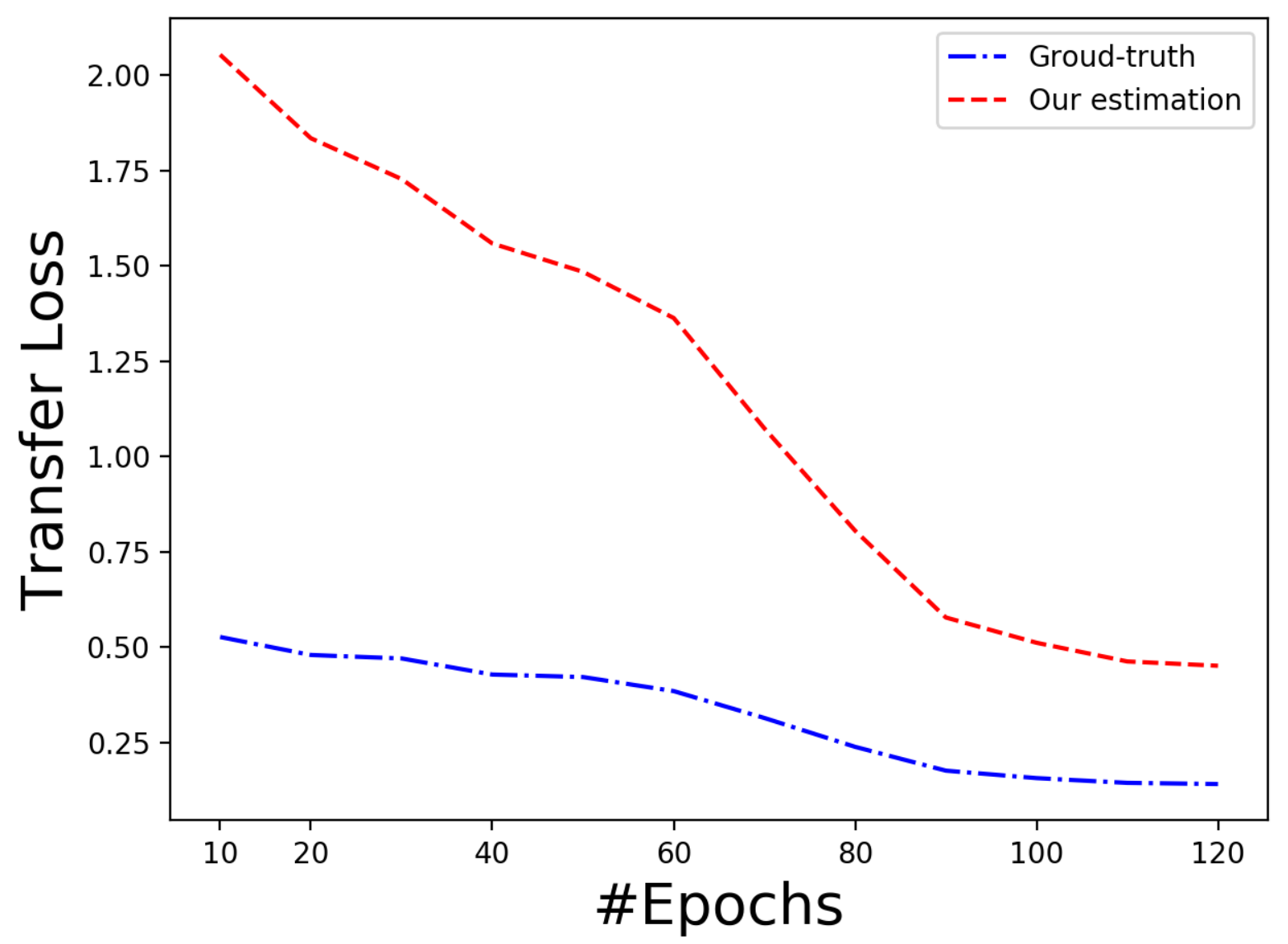}
\caption{$\|C_\S-C_\T\|_F/\|K_\T\|_F$ (i.e., our estimation) vs. $\|K_\S-K_\T\|_F/\|K_\T\|_F$ (i.e., ground-truth). \label{fig:corr}}
\end{figure}

Finally, we demonstrate that the difference between partial Gram matrices is closely correlated with that between full Gram matrices as suggested in Theorem~\ref{thm:3}. Features from the layer before the last FC layer are adopted for evaluation. Fig.~\ref{fig:corr} illustrates how the ground-truth transfer loss $\|K_\S-K_\T\|_F/\|K_\T\|_F$ and the estimated transfer loss $\|C_\S-C_\T\|_F/\|K_\T\|_F$ are changing during the training ($100\times$values used for better visualization). Evidently, minimizing $\|C_\S-C_\T\|_F$ can reduce the gap between the original Gram matrices effectively, which is consistent with our theoretical analysis. Note that we have an assumption in Theorem~\ref{thm:3} that the smallest eigenvalues of $W_{\S}$ and $W_{\T}$ are larger than $1$. Since we adopt class centers as landmark points, $W_{\S}$ and $W_{\T}$ can be full rank matrices. We show their smallest eigenvalues in Fig.~\ref{fig:eigen}. It is obvious that the smallest eigenvalue is larger than 10, consistent with our assumption.

\begin{figure}[!ht]
\centering
\includegraphics[width=2.4in]{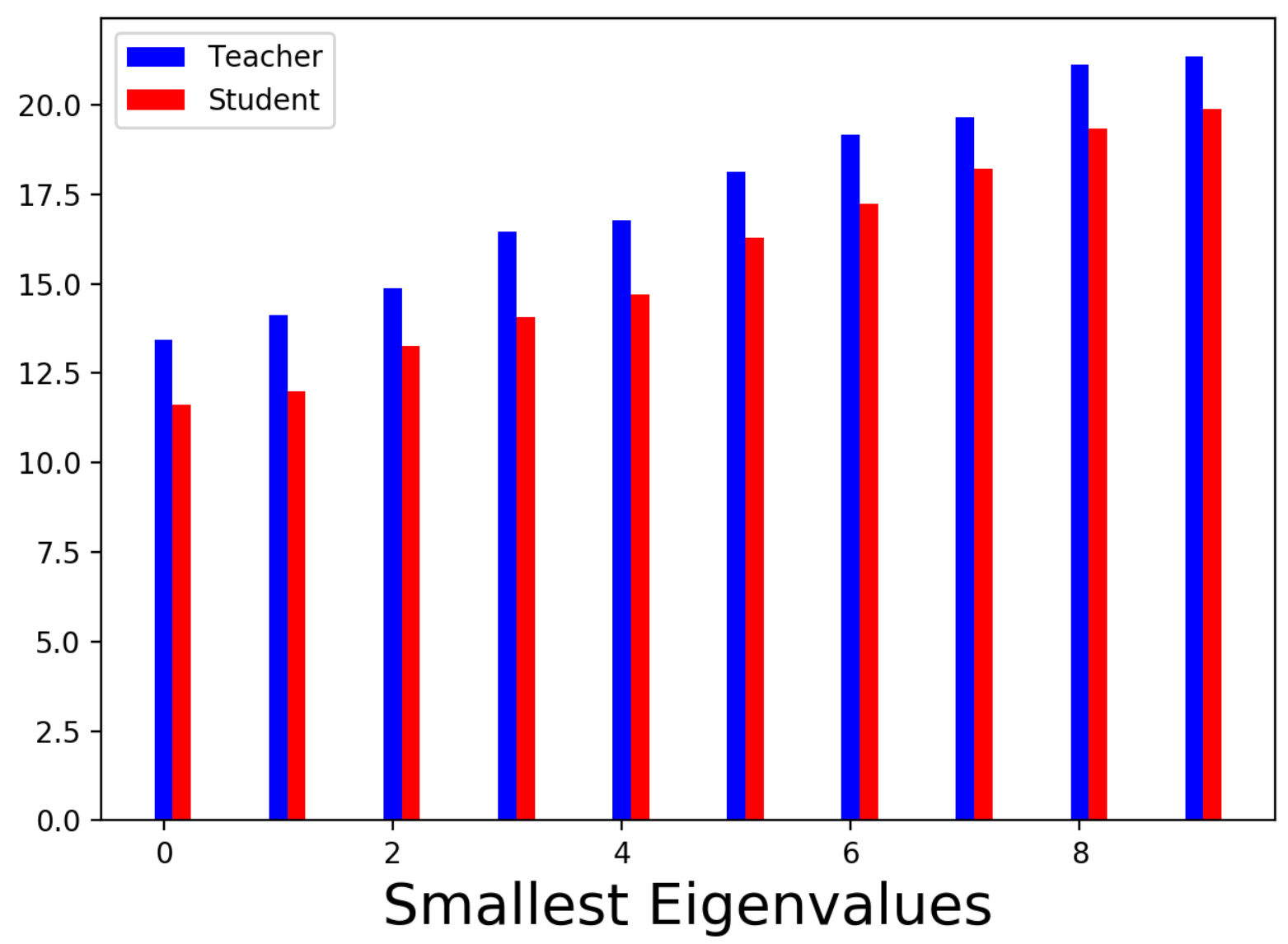}
\caption{Smallest eigenvalues of $W_{\S}$ (i.e., student) and $W_{\T}$ (i.e., teacher). \label{fig:eigen}}
\end{figure}

\subsubsection{Effect of Matrix $W$}

Corollary~\ref{cor:1} implies a variant that uses the standard Nystr{\"{o}}m method including matrix $W$ for transferring, while our proposal ignores $W$ and optimizes only the difference between $C_\S$ and $C_\T$ for efficiency. We compare our method to the one with $W$ in Fig.~\ref{fig:w}, where features before the last FC layer are adopted for transfer. During the experiment, we find that $W_\T^{+\frac{1}{2}}$ provides better performance. We can observe that our method without $W$ has a similar performance as the one with $W_\T^{+\frac{1}{2}}$. It further demonstrates our analysis in Theorem~\ref{thm:3}.

\begin{figure}[!ht]
\centering
\includegraphics[width=2.4in]{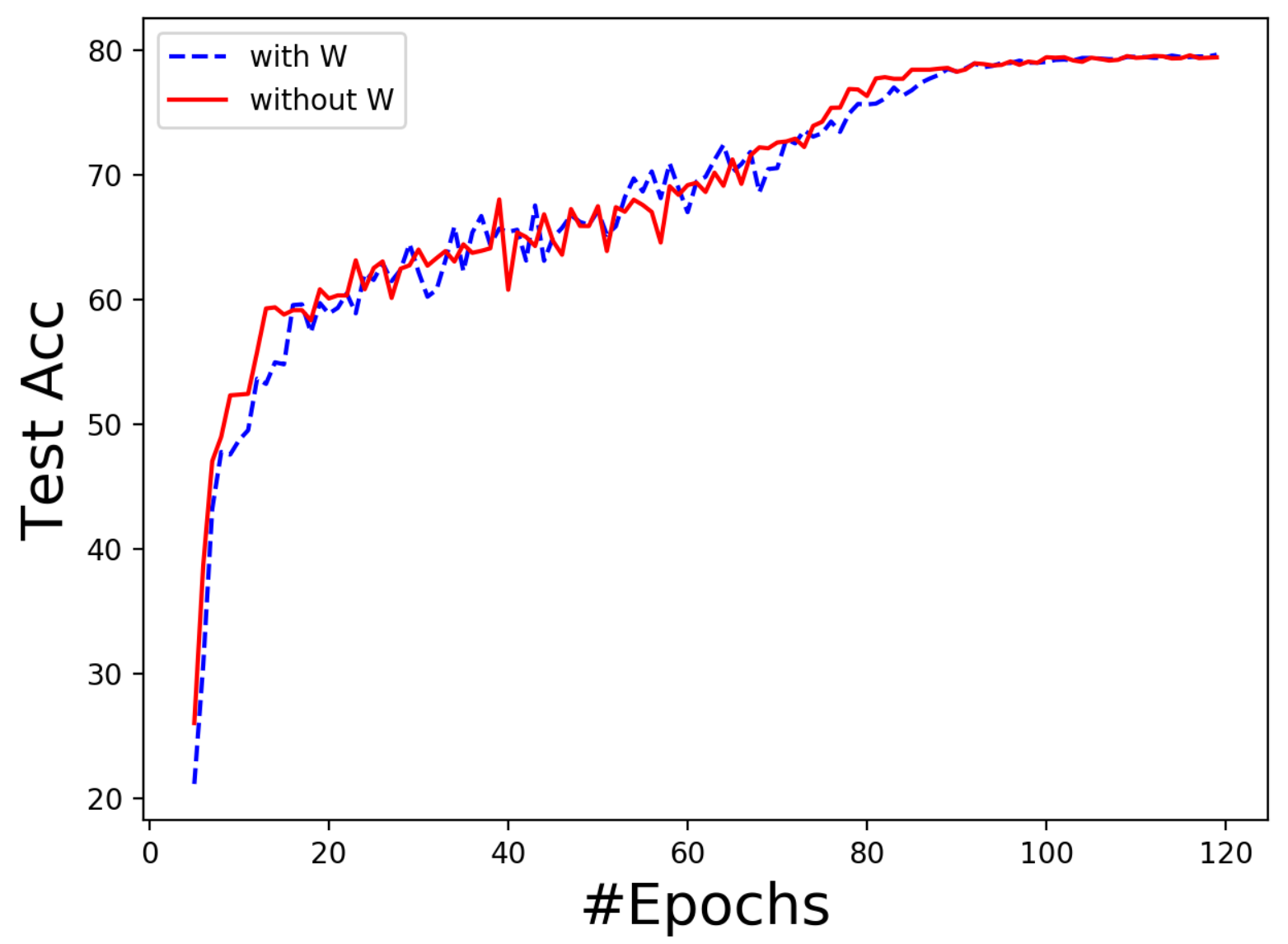}
\caption{Comparison of our method and the method with matrix $W$ as in Corollary~\ref{cor:1} \label{fig:w}}
\end{figure}

\subsubsection{Effect of $\#$Centers per Class}
When assigning landmark points, we set the number to be that of classes, which avoids clustering in the implementation. It also constrains that each class has a single landmark point. The number of landmark points for each class can be easily increased by clustering. We compare the variant with two centers per class in Fig.~\ref{fig:mc}, where features before the last FC layer are adopted for comparison.

\begin{figure}[!ht]
\centering
\includegraphics[width=2.4in]{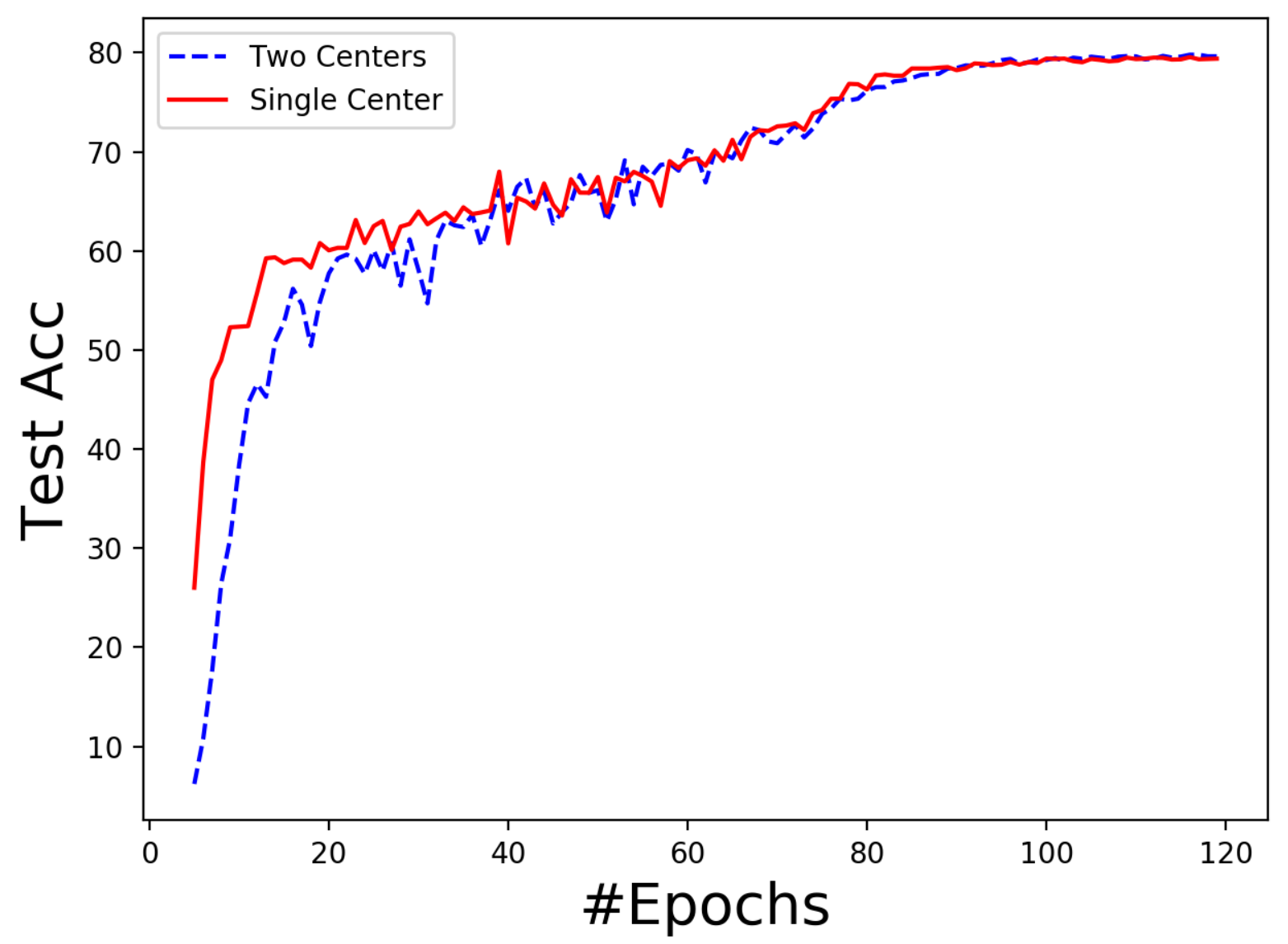}
\caption{Comparison of different number of landmark points (i.e., centers) for each class.\label{fig:mc}}
\end{figure}

We can observe that more centers for each class cannot improve the performance significantly. It may be because that the feature space is optimized with the cross entropy loss. As illustrated in \cite{QianSSHTLJ19}, cross entropy loss will push all examples from the same class to a single center. Therefore, assigning one landmark point for each class is an appropriate setting, which also simplifies the algorithm and improves the efficiency. We use one center per class in the following experiments.

\begin{table}[!ht]
\centering
\small
\caption{Comparison of accuracy ($\%$) on CIFAR-100 when transferring the full Gram matrix from different layers (S is the student without KD and T is the teacher).}\label{ta:mc}
\begin{tabular}{|l|l|l|l|l|l|}\hline
S&conv3\_x&conv4\_x&Before FC&After FC&T\\\hline
77.2&77.7&78.1&79.6&79.4&80.3\\\hline
\end{tabular}
\end{table}

\begin{table*}[!ht]
\centering
\footnotesize
\caption{Comparison of accuracy ($\%$) on CIFAR-100 (S is student without KD and T is teacher).}\label{ta:cifar100}
\resizebox{\textwidth}{10mm}{
\begin{tabular}{|l|l|l||l|l|l||l|l||l||l|}\hline
\multirow{2}{*}{T}&\multirow{2}{*}{S}&\multirow{2}{*}{S}&\multicolumn{3}{c||}{Before Last FC}&\multicolumn{2}{c||}{After FC}&Combo&\multirow{2}{*}{T}\\\cline{4-9}
&&&AT&RKD&KDA&KD&KDA&KDA&\\\hline
R34&R18&77.2&78.1$\pm$0.3&78.3$\pm$0.2&79.6$\pm$0.1&78.8$\pm$0.2&79.4$\pm$0.1&  \textbf{79.7$\pm$0.1} &80.3\\
R34&R18-0.5&73.5&75.0$\pm$0.1&74.3$\pm$0.3&75.6$\pm$0.3&74.8$\pm$0.2&75.3$\pm$0.2&\textbf{75.9$\pm$0.2}&80.3\\
R34&SN&71.7&73.0$\pm$0.1&72.5$\pm$0.1&74.0$\pm$0.3&72.9$\pm$0.1&73.6$\pm$0.1&\textbf{74.2$\pm$0.3}&80.3\\\hline
\end{tabular}}
\end{table*}

\begin{table*}[!ht]
\centering
\footnotesize
\caption{Comparison of accuracy ($\%$) on Tiny-ImageNet (S is student without KD and T is teacher).}\label{ta:tiny}
\resizebox{\textwidth}{10mm}{
\begin{tabular}{|l|l|l||l|l|l||l|l||l||l|}\hline
\multirow{2}{*}{T}&\multirow{2}{*}{S}&\multirow{2}{*}{S}&\multicolumn{3}{c||}{Before Last FC}&\multicolumn{2}{c||}{After FC}&Combo&\multirow{2}{*}{T}\\\cline{4-9}
&&&AT&RKD&KDA&KD&KDA&KDA&\\\hline
R34&R18&63.4&64.4$\pm$0.1&63.9$\pm$0.2&65.2$\pm$0.3&64.9$\pm$0.2&65.4$\pm$0.1&\textbf{65.5$\pm$0.1}&66.6\\
R34&R18-0.5&60.3&61.0$\pm$0.1&60.6$\pm$0.2&61.7$\pm$0.3&61.3$\pm$0.1&61.9$\pm$0.2&\textbf{62.2$\pm$0.2}&66.6\\
R34&SN&60.6&61.3$\pm$0.1&61.2$\pm$0.2&62.0$\pm$0.2&61.5$\pm$0.1&62.3$\pm$0.2&\textbf{62.4$\pm$0.1}&66.6\\\hline
\end{tabular}}
\end{table*}

\subsubsection{Effect of Different Layers}
Now, we illustrate the performance of transferring the Gram matrix from different layers. ResNet consists of $5$ convolutional layer groups and we compare the performance of the last $3$ groups (i.e., ``conv3\_x'', ``conv4\_x'' and ``conv5\_x'') and the one after the last FC layer. The definition of groups can be found in \cite{HeZRS16}. For each group, we adopt the last layer for transfer. Before transfer, we add a pooling layer to reduce the dimension of the feature map. Note that after pooling, the last layer of ``conv5\_x'' becomes the layer before the last FC layer.

Table~\ref{ta:mc} shows the performance of transferring information from different layers. First, transferring information from teacher always improves the performance of student, which demonstrates the effectiveness of knowledge distillation. Besides, the information from the later layers is more helpful for training student. It is because later layers contain more semantic information that is closely related to the target task. We will focus on the layers before and after the FC layer in the rest experiments.

\subsection{CIFAR-100}

In this subsection, we compare the proposed KDA method to baseline methods on CIFAR-100. The results of different methods can be found in Table~\ref{ta:cifar100}. All experiments are repeated $3$ times and the average results with standard deviation are reported.

First, knowledge distillation methods outperform training the student network without a teacher, which shows that knowledge distillation can improve the performance of student models significantly. By transferring the full Gram matrix before the last FC layer, KDA surpasses RKD by $1.3\%$ when ResNet-34 is the teacher and ResNet-18 is the student. The observation is consistent with the comparison in the ablation study, which confirms that $\|K_\S-K_\T\|_F/\|K_\T\|_F$ is an appropriate metric to evaluate the transfer loss of the Gram matrix. Moreover, KDA shows a significant improvement on different student networks, which further demonstrates the applicability of the proposed method.

Furthermore, when transferring the Gram matrix after the last FC layer, both of KD and KDA can demonstrate good performance compared with the student model. It is due to the fact that these methods transfer the Gram matrix with landmark points, which is efficient for optimization. Besides, KDA can further improve the performance compared to KD. The superior performance of KDA demonstrates the effectiveness of the proposed strategy for generating appropriate landmark points.

Finally, compared to benchmark methods, KDA can distill the information from different layers with the same formulation in Eqn.~\ref{eq:kloss}. The proposed method provides a systematic perspective to understand a family of knowledge distillation methods that aim to transfer Gram matrix. If transferring the Gram matrices before and after the FC layer simultaneously, the performance of KDA can be slightly improved as illustrated by ``Combo'' in Table~\ref{ta:cifar100}.

\subsection{Tiny-ImageNet}
Then, we compare different methods on Tiny-ImageNet data set\footnote{https://tiny-imagenet.herokuapp.com}. There are $200$ classes in this data set and each class provides $500$ images for training and $50$ for validation. We report the performance on the validation set. Since the size of images in Tiny-ImageNet is $64\times 64$ that is larger than CIFAR-100, we replace the random crop augmentation with a more aggressive version as in \cite{HeZRS16}, and keep other settings the same.

Table~\ref{ta:tiny} summarizes the comparison. We can observe the similar results as on CIFAR-100. First, all methods with the information from a teacher model can surpass the student model without a teacher by a large margin. Second, KDA outperforms other methods no matter in which layer the transfer happens. Note that CIFAR-100 and Tiny-ImageNet have very different images, which demonstrates the applicability of the proposed algorithm in different real-world applications.

\section{Conclusions}
\label{sec:conclude}
In this work, we investigate the knowledge distillation problem from the perspective of kernel matrix transfer. Considering the number of terms in the full kernel matrix is quadratic in the number of training examples, we extend the Nystr{\"{o}}m method and propose a strategy to obtain the landmark points for efficient transfer. The proposed method not only improves the efficiency of transferring the kernel matrix, but also has the theoretical guarantee for the efficacy. Experiments on the benchmark data sets verify the effectiveness of the proposed algorithm. 

Besides the similarity function applied in this work, there are many other complicated functions adopted by other methods. Combining the proposed algorithm with different similarity functions for distillation can be our future work.


\bibliographystyle{siam}
\bibliography{distill19}

\end{document}